\documentclass{article} %

\usepackage{graphicx}
\usepackage{subfigure}
\usepackage{natbib}
\usepackage{etoolbox}%
\usepackage{relsize} %

\usepackage[noend]{algorithmic}
\usepackage{algorithm}

\usepackage[accepted]{icml2013}

\makeatletter
\renewcommand{\ICML@appearing}{ This is an extend version of the paper
  of the same name which appeared in ICML 2013.  The main addition is
  Appendix A.3, which contains additional proofs. }
\makeatother

\ifdef{\thisisicml}{%
\newcommand{\icmlonly}[1]{#1}
\newcommand{\fullonly}[1]{}
}{%
\newcommand{\icmlonly}[1]{}
\newcommand{\fullonly}[1]{#1}
}

\usepackage{array}
\usepackage{amsfonts}
\usepackage{amssymb}
\usepackage{amstext}
\usepackage{amsmath}
\usepackage{xspace}
\usepackage{amsthm}
\usepackage{epsfig}
\usepackage{color}
\usepackage{url}

\DeclareMathOperator*{\argmin}{arg\,min}

\DeclareMathOperator{\var}{Var}
\DeclareMathOperator{\cov}{Cov}
\DeclareMathOperator{\Erf}{Erf}

\definecolor{darkblue}{rgb}{0,0,0.7}

\newcommand{\defeq}{\equiv}  %
\newcommand{\grad}{\nabla}

\newcommand{\set}[1] {\ensuremath{ \{ {#1} \} }}

\newcommand{\R}[0]{{\ensuremath{\mathbb{R}}}}
\newcommand{\Z}[0]{{\ensuremath{\mathbb{Z}}}}
\newcommand{\F}{\mathcal{F}}

\newcommand{\ti}{_{t+1}}

\newcommand{\hb}{\hat{\beta}}
\newcommand{\bs}{\beta^*}
\newcommand{\eps}{\epsilon}

\newcommand{\abs}[1]{|#1|}
\newcommand{\babs}[1]{\left|#1\right|}
\newcommand{\eqr}[1]{Eq.~\eqref{eq:#1}}

\newcommand{\BO}{\mathcal{O}}
\newcommand{\gam}{\gamma}
\newcommand{\h}{\frac{1}{2}}

\newcommand{\cnt}{\tau} %
\newcommand{\acnt}{\tilde{\tau}}  %
\newcommand{\cntTi}{\mathlarger{\cnt}_{\mbox{\tiny\itshape T}, i}}
\newcommand{\est}{\tilde{\tau}}

\newcommand{\paren} [1] {\ensuremath{ \left( {#1} \right) }}
\newcommand{\E}{\mathbf{E}}
\newcommand{\prob}[1]{\ensuremath{\text{{\bf Pr}$\left[#1\right]$}}}
\newcommand{\expct}[1]{\E \left[#1\right]}
\newcommand{\ceil}[1]{\ensuremath{\left\lceil#1\right\rceil}}
\newcommand{\floor}[1]{\ensuremath{\left\lfloor#1\right\rfloor}}

\renewcommand{\dim} {\ensuremath{d}}
\newcommand{\norm}[2] {\|#1\|_{#2}}
\newcommand{\nrm}[1] {\|#1\|}

\newtheorem{lemma}{Lemma}[section]
\newtheorem{theorem}[lemma]{Theorem}
\newtheorem{corollary}[lemma]{Corollary}

\newcommand{\invlogit}[1]{\ensuremath{\sigma\!\paren{#1}}}

\newcommand{\predict}[1]{\ensuremath{p_{\beta}\paren{#1}}}
\newcommand{\logloss}[1]{\ensuremath{\mathcal{L}(#1)}}

\newcommand{\proj}[1]{\ensuremath{\operatorname{Project}\paren{{#1}}}}
\newcommand{\Proj}{\operatorname{Project}}

\newcommand{\Regret}{\operatorname{Regret}}

\icmltitlerunning{Large-Scale Learning with Less RAM via Randomization}

\begin{document}

\twocolumn[
\icmltitle{Large-Scale Learning with Less RAM via Randomization}

\icmlauthor{Daniel Golovin}{dgg@google.com}

\icmlauthor{D. Sculley}{dsculley@google.com}

\icmlauthor{H. Brendan McMahan}{mcmahan@google.com}

\icmlauthor{Michael Young}{mwyoung@google.com}

\icmladdress{Google, Inc., Pittsburgh, PA, and Seattle, WA}

\icmlkeywords{online learning, online gradient descent, big learning}

\vskip 0.2in
]

\begin{abstract}
  We reduce the memory footprint of popular large-scale online
  learning methods by projecting our weight vector onto a coarse
  discrete set using randomized rounding.  Compared to standard 32-bit
  float encodings, this reduces RAM usage by more than 50\% during
  training and by up to 95\% when making predictions from a fixed model,
  with almost no loss in accuracy.  We also show that randomized
  counting can be used to implement per-coordinate learning rates,
  improving model quality with little additional RAM.
  We prove these memory-saving methods achieve regret guarantees
  similar to their exact variants.  Empirical evaluation confirms
  excellent performance, dominating standard approaches across
  memory versus accuracy tradeoffs.
\end{abstract}

\section{Introduction}
As the growth of machine learning data sets continues to accelerate,
available machine memory (RAM) is an increasingly important
constraint.  This is true for training massive-scale distributed
learning systems, such as those used for predicting ad click through
rates (CTR) for sponsored
search~\citep{richardson:2007,craswell:2008,bilenko:2011,streeter:2010}
or for filtering email spam at scale~\citep{goodman:2007}.  Minimizing
RAM use is also important on a single machine if we wish to utilize
the limited memory of a fast GPU processor, or to simply use fast
L1-cache more effectively.  After training, memory cost remains a key
consideration at prediction time as real-world models are often
replicated to multiple machines to minimize prediction latency.

Efficient learning at peta-scale is commonly achieved by online
gradient descent (OGD)~\citep{zinkevich03} or stochastic gradient
descent (SGD), \citep[e.g.,][]{bottou:2008}, in which many tiny steps
are accumulated in a weight vector $\beta \in \R^{\dim}$.  For
large-scale learning, storing $\beta$ can consume considerable
RAM, especially when datasets far exceed memory
capacity and examples are streamed from network or disk.

Our goal is to reduce the memory needed to store $\beta$.
Standard implementations store coefficients in single precision
floating-point representation, using 32 bits per value.
This provides fine-grained precision needed to accumulate these tiny
steps with minimal roundoff error, but has a dynamic range that far
exceeds the needs of practical machine learning (see Figure~\ref{weights}).

We use coefficient representations that have more limited precision
and dynamic range, allowing values to be stored cheaply.  This coarse
grid does {\em not} provide enough resolution to accumulate gradient
steps without error, as the grid spacing may be larger than the
updates.  But we can obtain a provable safety guarantee through a
suitable OGD algorithm that uses randomized rounding to project its
coefficients onto the grid each round.  The precision of the grid used
on each round may be fixed in advance or changed adaptively as
learning progresses.  At prediction time, more aggressive rounding is
possible because errors no longer accumulate.

Online learning on large feature spaces where some features occur very
frequently and others are rare often benefits from per-coordinate
learning rates, but this requires an additional 32-bit count to be
stored for each coordinate.  In the spirit of randomized rounding, we
limit the memory footprint of this strategy by using an 8-bit
randomized counter for each coordinate based on a variant of Morris's
algorithm (1978).  We show the resulting regret bounds are only
slightly worse than the exact counting variant
(Theorem~\ref{thm:ogd-approx-count}), and empirical results show
negligible added loss.

\paragraph{Contributions}
This paper gives the following theoretical and empirical results:

\vspace{-0.1in}
\begin{enumerate} \itemsep -1.2pt
\item Using a pre-determined fixed-point representation of coefficient
  values reduces cost from 32 to 16 bits per value, at the cost of a
  small linear regret term.
\item The cost of a per-coordinate learning rate schedule can be
  reduced from 32 to 8 bits per coordinate using a randomized counting
  scheme.
\item Using an adaptive per-coordinate coarse representation of
  coefficient values reduces memory cost further and yields a
  no--regret algorithm.
\item Variable-width encoding at prediction time allows coefficients
  to be encoded even more compactly (less than 2 bits per value in
  experiments) with negligible added loss.
\end{enumerate}
\vspace{-0.1in} Approaches 1 and 2 are particularly attractive, as
they require only small code changes and use negligible additional CPU
time.  Approaches 3 and 4 require more sophisticated data structures.

\begin{figure}
\begin{centering}
\includegraphics[width=2.5in]{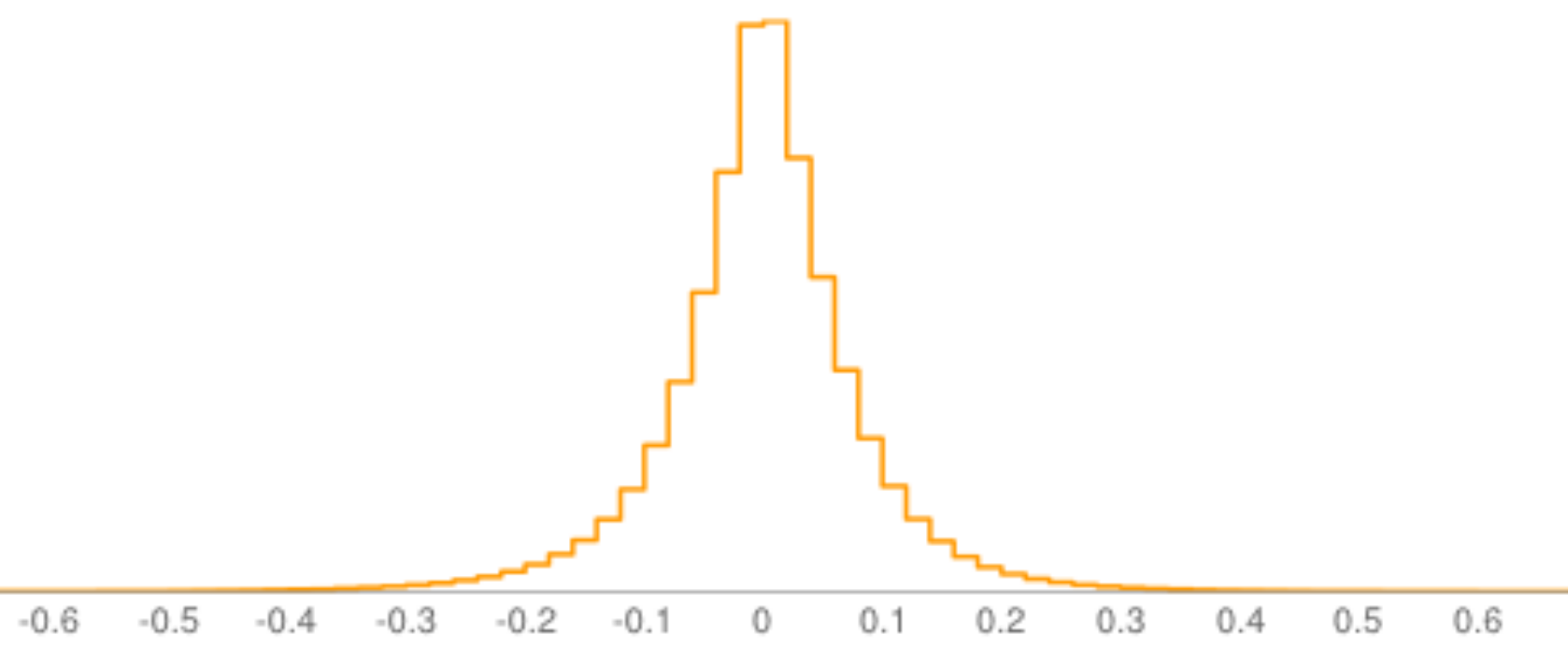}
\vspace{-0.05in}
\caption{Histogram of coefficients in a typical large-scale linear
  model trained from real data.
  Values are tightly grouped near zero; a large dynamic range is
  superfluous.}
\vspace{-0.15in}
\end{centering}
\label{weights}
\end{figure}

\section{Related Work}
In addition to the sources already referenced, related work has been
done in several areas.

\paragraph{Smaller Models}
A classic approach to reducing memory usage is to encourage sparsity,
for example via the Lasso~\cite{tibshirani} variant of least-squares
regression, and the more general application of $L_1$
regularizers~\cite{duchi:2008,langford:2009,xiao09dualaveraging,mcmahan:2011}.
A more recent trend has been to reduce memory cost via the use of
feature hashing~\cite{weinberger:2009}. Both families of approaches
are effective.  The coarse encoding schemes reported here may be used
in conjunction with these methods to give further reductions in memory
usage.

\paragraph{Randomized Rounding}
Randomized rounding schemes have been widely used in numerical
computing and algorithm design~\cite{raghavan:1987}.  Recently, the
related technique of randomized counting has enabled compact language
models \cite{vandurme:2009}.  To our knowledge, this
paper gives the first algorithms and analysis for online learning with
randomized rounding and counting.
\vspace{-0.1in}

\paragraph{Per-Coordinate Learning Rates}
\citet{duchi10adaptive} and \citet{mcmahan10boundopt} demonstrated
that per-coordinate adaptive regularization ({\em i.e.}, adaptive
learning rates) can greatly boost prediction accuracy.  The intuition
is to let the learning rate for common features decrease quickly,
while keeping the learning rate high for rare features.  This
adaptivity increases RAM cost by requiring an additional statistic to
be stored for each coordinate, most often as an additional 32-bit
integer.  Our approach reduces this cost by using an 8-bit
randomized counter instead, using a variant of Morris's
algorithm~\citep{morris78}.

\section{Learning with Randomized Rounding
and Probabilistic Counting}
\label{sec:training-time-results}

For concreteness, we focus on logistic regression with binary feature
vectors $x \in \set{0,1}^{\dim}$ and labels $y \in \set{0,1}$.  The
model has coefficients $\beta \in \R^{\dim}$, and gives predictions
$\predict{x} \defeq \invlogit{\beta \cdot x}$, where $\invlogit{z} \defeq
1/(1+e^{-z})$ is the logistic function.
Logistic regression finds the model that minimizes the logistic--loss
$\mathcal{L}$.  Given a labeled example $(x, y)$
the logistic--loss is
\[
 \logloss{x, y; \beta} \defeq - y \log \paren{\predict{x}}
- (1-y) \log \paren{1 - \predict{x}}
\]
where we take $0 \log 0 = 0$.  Here, we take $\log$ to be the natural
logarithm.
We define $\norm{x}{p}$ as the $\ell_p$ norm of a vector $x$;
when the subscript $p$ is omitted, the $\ell_2$ norm is implied.  We
use the compressed summation notation $g_{1:t} \defeq \sum_{s=1}^t
g_s$ for scalars, and similarly $f_{1:t}(x) \defeq \sum_{s=1}^t
f_s(x)$ for functions.

The basic algorithm we propose and analyze is a
variant of online gradient descent (OGD) that stores
coefficients $\beta$ in a limited precision format using
a discrete set $(\epsilon \Z)^{\dim}$.  For each OGD update, we
compute each new coefficient value in 64-bit floating point representation and
then use randomized rounding to project the updated value back to
the coarser representation.

A useful representation for the discrete set $(\epsilon \Z)^{\dim}$ is
the {\tt Qn.m} fixed-point representation.  This uses {\tt n} bits for
the integral part of the value, and {\tt m} bits for the fractional
part.  Adding in a sign bit results in a total of $K={\tt n} + {\tt m}
+ 1$ bits per value.  The value {\tt m} may be fixed in advance, or
set adaptively as described below.  We use the method RandomRound from
Algorithm~\ref{alg:ogdone} to project values onto this encoding.

The added CPU cost of fixed-point encoding and randomized rounding is
low.  Typically $K$ is chosen to correspond to a machine integer (say
$K=8$ or $16$), so converting back to a floating point representations
requires a single integer-float multiplication (by $\epsilon = 2^{-m}$).
Randomized rounding requires a call to a pseudo-random number
generator, which may be done in 18-20 flops.  Overall, the added CPU
overhead is negligible, especially as many large-scale learning
methods are I/O bound reading from disk or network rather than CPU
bound.

\begin{algorithm}[t]
\caption{OGD-Rand-1d } \label{alg:ogdone}
\begin{small}
\begin{algorithmic}
   \STATE \textbf{input:} feasible set $\F = [-R, R]$,
     learning rate schedule $\eta_t$, resolution schedule $\eps_t$
   \STATE \textbf{define fun} $\proj{\beta} = \max(-R, \min(\beta, R))$
   \STATE Initialize $\hb_1 = 0$
   \FOR{t=1, \dots, T}
     \STATE Play the point $\hb_t$, observe $g_t$
     \STATE $\beta\ti = \Proj\big(\hb_t - \eta_t g_t\big)$
     \STATE $\hb\ti \leftarrow \text{RandomRound}(\beta\ti, \eps_t)$
   \ENDFOR
   \STATE
   \FUNCTION{RandomRound$(\beta, \eps)$}{}
   \STATE $a \leftarrow \eps\floor{\frac{\beta}{\eps}}$;
          $b \leftarrow \eps\ceil{\frac{\beta}{\eps}}$
   \STATE return $\begin{cases}
     b & \text{with prob.}\quad (\beta - a)/\eps\\
     a & \text{otherwise}
   \end{cases}$
   \ENDFUNCTION
\end{algorithmic}
\vspace{-0.1in}
\end{small}
\end{algorithm}

\subsection{Regret Bounds for Randomized Rounding}
We now prove theoretical guarantees (in the form of upper bounds on
regret) for a variant of OGD that uses randomized rounding on an
adaptive grid as well as per-coordinate learning rates.  (These bounds
can also be applied to a fixed grid).  We use the standard definition
\[
\Regret \defeq \sum_{t=1}^T f_t(\hb_t) -
               \argmin_{\bs \in \F} \sum_{t=1}^T f_t(\bs)
\]
given a sequence of convex loss functions $f_t$.  Here the $\hb_t$ our
algorithm plays are random variables, and since we allow the adversary
to adapt based on the previously observed $\hb_t$, the $f_t$ and
post-hoc optimal $\bs$ are also random variables.
We prove bounds on \emph{expected} regret, where the expectation is
with respect to the randomization used by our algorithms
(high-probability bounds are also possible).
We consider regret with respect to the best model in the
\emph{non-discretized} comparison class $\F = [-R, R]^\dim$.

We follow the usual reduction from convex to linear functions
introduced by \citet{zinkevich03}; see also
\citet[Sec. 2.4]{shwartz12online}.  Further, since we consider the
hyper-rectangle feasible set $\F = [-R, R]^\dim$, the linear problem
decomposes into $n$ independent one-dimensional
problems.\footnote{Extension to arbitrary feasible sets is possible,
  but choosing the hyper-rectangle simplifies the analysis; in
  practice, projection onto the feasible set rarely helps performance.
} In this setting, we consider OGD with randomized rounding to an
adaptive grid of resolution $\eps_t$ on round $t$, and an adaptive
learning rate $\eta_t$.  We then run one copy of this algorithm for
each coordinate of the original convex problem, implying that we can
choose the $\eta_t$ and $\eps_t$ schedules appropriately for each
coordinate.  For simplicity, we assume the $\eps_t$ resolutions are
chosen so that $-R$ and $+R$ are always gridpoints.
Algorithm~\ref{alg:ogdone} gives the one-dimensional version, which is
run independently on each coordinate (with a different learning rate
and discretization schedule) in Algorithm~\ref{alg:ogdfull}.  The core
result is a regret bound for Algorithm~\ref{alg:ogdone} (omitted
proofs can be found in the Appendix):

\begin{theorem}\label{thm:ogd}
  Consider running Algorithm~\ref{alg:ogdone} with adaptive
  non-increasing learning-rate schedule $\eta_t$, and discretization
  schedule $\eps_t$ such that $\eps_t \leq \gam \eta_t$ for a
  constant $\gam > 0$. Then, against any sequence of gradients $g_1,
  \dots, g_T$ (possibly selected by an adaptive adversary) with
  $\abs{g_t} \leq G$, against any comparator point $\bs \in [-R, R]$,
  we have
  \[
  \E[\Regret(\bs)] \le \frac{(2R)^2}{2 \eta_T} + \h (G^2 +
  \gam^2)\eta_{1:T} + \gam R \sqrt{T}.
  \]
\end{theorem}

By choosing $\gamma$ sufficiently small, we obtain an expected regret
bound that is indistinguishable from the non-rounded version (which is
obtained by taking $\gamma = 0$).  In practice, we find simply
choosing $\gamma = 1$ yields excellent results.  With some care in the
choice of norms used, it is straightforward to extend the above result
to $\dim$ dimensions.  Applying the above algorithm on a
per-coordinate basis yields the following guarantee:

\begin{corollary}\label{cor:full}
  Consider running Algorithm~\ref{alg:ogdfull} on the feasible set $\F
  = [-R, R]^\dim$, which in turn runs Algorithm~\ref{alg:ogdone} on
  each coordinate. We use per-coordinate learning rates $\eta_{t,i} =
  \alpha/\sqrt{\cnt_{t,i}}$ with $\alpha = \sqrt{2} R / \sqrt{G^2 + \gam^2}$,
  where $\cnt_{t,i} \le t$ is the number of non-zero $g_{s,i}$ seen on
  coordinate $i$ on rounds $s=1, \dots, t$.
  Then, against convex loss functions $f_t$, with $g_t$ a subgradient
  of $f_t$ at $\hb_t$, such that $\forall t, \ \norm{g_t}{\infty} \leq
  G$, we have
  \[ \E[\Regret] \le \sum_{i=1}^\dim\left(2R \sqrt{2 \cntTi (G^2 + \gam^2)}
       + \gam R \sqrt{\cntTi}\right).\]
\end{corollary}
The proof follows by summing the bound from Theorem~\ref{thm:ogd} over
each coordinate, considering only the rounds when $g_{t,i} \neq 0$,
and then using the inequality $\sum_{t=1}^T 1/\sqrt{t} \leq 2
\sqrt{T}$ to handle the sum of learning rates on each coordinate.

The core intuition behind this algorithm is that for features where we
have little data (that is, $\cnt_i$ is small, for example rare words
in a bag-of-words representation, identified by a binary feature),
using a fine-precision coefficient is unnecessary, as we can't
estimate the correct coefficient with much confidence. This is in fact
the same reason using a larger learning rate is appropriate, so it is
no coincidence the theory suggests choosing $\eps_t$ and $\eta_t$ to
be of the same magnitude.

\begin{algorithm}[t]
\caption{OGD-Rand} \label{alg:ogdfull}
\begin{small}
\begin{algorithmic}
   \STATE \textbf{input:} feasible set $\F = [-R, R]^\dim$,
   parameters $\alpha, \gamma > 0$
   \STATE Initialize $\hb_1 = 0 \in \R^\dim;\ \forall i, \cnt_i = 0$
   \FOR{t=1, \dots, T}
     \STATE Play the point $\hb_t$, observe loss function $f_t$
     \FOR{i=1, \dots, $\dim$}
       \STATE let $g_{t,i} = \grad f_t(x_t)_i$
       \STATE \textbf{if} $g_{t,i} = 0$ \textbf{then} continue
       \STATE $\cnt_i \leftarrow \cnt_i + 1$
       \STATE let $\eta_{t,i} = \alpha/\sqrt{\cnt_i}$
          and $\eps_{t,i} = \gam \eta_{t,i}$
       \STATE $\beta_{t+1,i} \leftarrow
          \Proj\big(\hb_{t,i} - \eta_{t,i} g_{t,i}\big)$
       \STATE $\hb_{t+1,i} \leftarrow \text{RandomRound}(
          \beta_{t+1,i}, \eps_{t,i})$
     \ENDFOR
   \ENDFOR
\end{algorithmic}
\end{small}
\end{algorithm}

\paragraph{Fixed Discretization} Rather than implementing an adaptive
discretization schedule, it is more straightforward and more efficient
to choose a fixed grid resolution, for example a 16-bit \texttt{Qn.m}
representation is sufficient for many applications.\footnote{If we
  scale $x \rightarrow 2x$ then we must take $\beta \rightarrow
  \beta/2$ to make the same predictions, and so appropriate choices of
  \texttt{n} and \texttt{m} must be data-dependent.}  In this case,
one can apply the above theory, but simply stop decreasing the
learning rate once it reaches say $\eps$ ($=2^{-{\tt m}}$).  Then, the
$\eta_{1:T}$ term in the regret bound yields a linear term like
$\BO(\eps T)$; this is unavoidable when using a fixed resolution
$\eps$.
One could let the learning rate continue to decrease like
$1/\sqrt{t}$, but this would provide no benefit; in fact,
lower-bounding the learning-rate is known to allow online gradient
descent to provide regret bounds against a moving comparator
\citep{zinkevich03}.

\paragraph{Data Structures}
There are several viable approaches to storing models with
variable--sized coefficients.  One can store all keys at a fixed
(low) precision, then maintain a sequence of maps ({\em e.g.}, as
hash-tables), each containing a mapping from keys to coefficients of
increasing precision.  Alternately, a simple linear probing
hash--table for variable length keys is efficient
for a wide variety of distributions on key lengths, as demonstrated
by~\citet{thorup:2009}.  With this data structure, keys and
coefficient values can be treated as strings over $4$-bit or $8$-bit
bytes, for example.  \citet{blandford:2008} provide yet another data
structure: a compact dictionary for variable length keys.  Finally,
for a fixed model, one can write out the string $s$ of all
coefficients (without end of string delimiters), store a second binary
string of length $s$ with ones at the coefficient boundaries, and use
any of a number of rank/select data structures to index into it, {\em e.g.},
the one of~\citet{patrascu:2008}.

\subsection{Approximate Feature Counts}\label{sec:approxcounts}
\newcommand{\base}[0]{\ensuremath{b}}
\newcommand{\counter}[0]{\ensuremath{C}}
\newcommand{\incprob}[0]{\ensuremath{p}}

Online convex optimization methods typically use a learning rate that
decreases over time, {\em e.g.}, setting $\eta_{t}$ proportional to
$1/\sqrt{t}$.  Per-coordinate learning rates require storing a unique
count $\cnt_i$ for each coordinate, where $\cnt_i$ is the number of
times coordinate $i$ has appeared with a non-zero gradient so far.
Significant space is saved by using a 8-bit randomized counting scheme
rather than a 32-bit (or 64-bit) integer to store the $d$ total
counts.
We use a variant of Morris' probabilistic counting algorithm (1978)
analyzed by~\citet{flajolet:1985}.  Specifically, we initialize a
counter $\counter = 1$, and on each increment operation, we increment
$\counter$ with probability $\incprob(\counter) = \base^{-\counter}$,
where base $\base$ is a parameter.  We estimate the count as
$\est(\counter) = \frac{\base^{\counter} - \base}{\base - 1}$, which
is an unbiased estimator of the true count.  We then use learning
rates $\eta_{t,i} = \alpha/\sqrt{\acnt_{t,i} + 1}$, which ensures that
even when $\acnt_{t,i} = 0$ we don't divide by zero.

We compute high-probability bounds on this counter in
Lemma~\ref{lem:approx-counting}.  Using these bounds for $\eta_{t,i}$
in conjunction with Theorem~\ref{thm:ogd}, we obtain the following
result (proof deferred to the appendix).

\begin{theorem}\label{thm:ogd-approx-count}
  Consider running the algorithm of Corollary~\ref{cor:full} under the
  assumptions specified there, but using approximate counts $\acnt_i$
  in place of the exact counts $\cnt_i$.  The approximate counts are
  computed using the randomized counter described above with any base
  $\base > 1$.  Thus, $\acnt_{t,i}$ is the estimated number of times
  $g_{s,i} \ne 0$ on rounds $s=1, \dots, t$, and the per--coordinate
  learning rates are $\eta_{t,i} = \alpha/\sqrt{\acnt_{t,i} + 1}$.  With
  an appropriate choice of $\alpha$ we have
  \[
  \E[\Regret(g)]=  o\paren{R \sqrt{G^2 + \gamma^2} T^{0.5 + \delta}}
     \quad \text{ for all } \delta > 0,
  \]
  where the $o$-notation hides a small constant factor and the
  dependence on the base $\base$.\footnote{\eqr{fullregret} in the
    appendix provides a non-asymptotic (but more cumbersome)
    regret bound.}
\end{theorem}

\begin{figure*}
\begin{centering}
\vspace{-0.1in}
\includegraphics[width=6.0in]{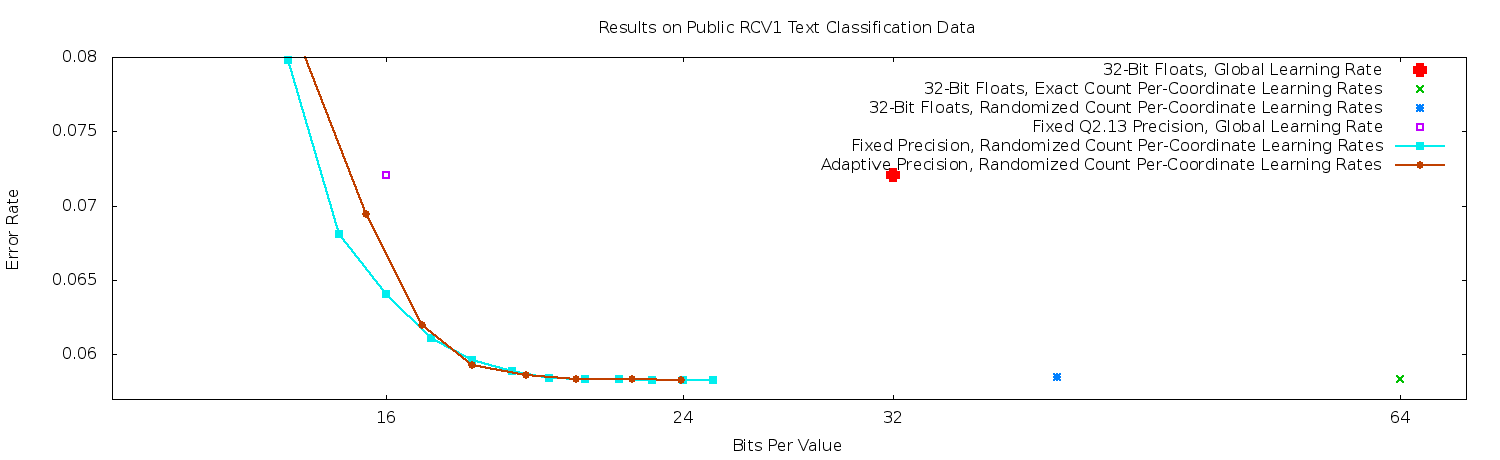}
\includegraphics[width=6.0in]{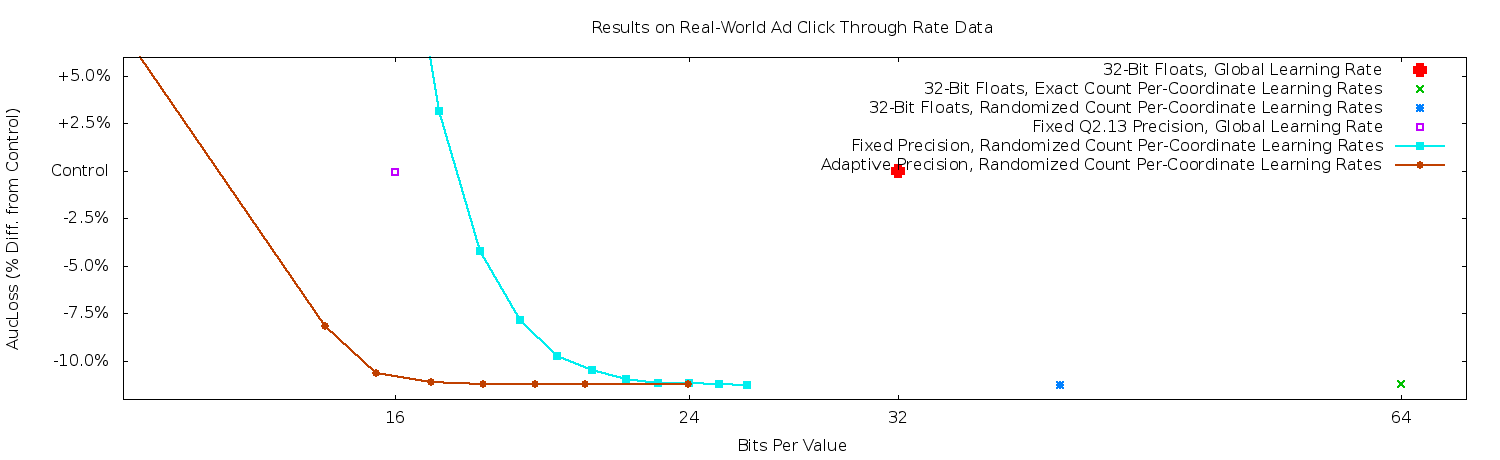}
\vspace{-0.2in}
\caption{ Rounding at Training Time.  The fixed {\tt q2.13} encoding
  is 50\% smaller than control with no loss.  Per-coordinate learning
  rates significantly improve predictions but use 64 bits per value.
  Randomized counting reduces this to 40 bits.  Using adaptive or
  fixed precision reduces memory use further, to 24 total bits per
  value or less.  The benefit of adaptive precision is seen more on
  the larger CTR data. } \label{fig:training-results}
\end{centering}
\vspace{-0.15in}
\end{figure*}

\section{Encoding During Prediction Time} \label{sec:prediction}
Many real-world problems require large-scale {\em prediction}.
Achieving scale may require that a trained model be replicated
to multiple machines~\cite{bucilua:2006}.  Saving RAM via rounding
is especially attractive here, because unlike in training
accumulated roundoff error is no longer an issue.
This allows even more aggressive rounding to be used safely.

Consider a rounding a trained model $\beta$ to some $\hat{\beta}$.  We
can bound both the additive and relative effect on logistic--loss
$\logloss{\cdot}$ in terms of the quantity $|\beta \cdot x -
\hat{\beta} \cdot x|$:
\begin{lemma}[Additive Error] \label{lem:fixed-error} Fix $\beta,
  \hat{\beta}$ and $(x,y)$.  Let $\delta = |\beta \cdot x -
  \hat{\beta} \cdot x|$.  Then the logistic--loss satisfies
  \[
    \logloss{x,y;\hat{\beta}} - \logloss{x,y;\beta} \le \delta.
  \]
\end{lemma}
\vspace{-0.1in}

\begin{proof}
  It is well known that $\left| \frac{\partial
      \logloss{x,y;\beta}}{\partial \beta_i} \right| \le 1$ for all
  $x,y,\beta$ and $i$, which implies the result.
\end{proof}

\begin{lemma}[Relative Error] \label{lem:relative-error}
Fix $\beta, \hat{\beta}$ and $(x,y) \in \set{0,1}^{\dim} \times \set{0,1}$.
Let $\delta = |\beta \cdot x - \hat{\beta} \cdot x|$.
Then
\[
\frac{\logloss{x,y;\hat{\beta}} - \logloss{x,y;\beta}}{\logloss{x,y;\beta}}
  \le e^\delta - 1.
\]
\end{lemma}
\newcommand{\hz}{\hat{z}} \icmlonly{Proofs for results in this
  section can be found in the extended version of this paper.
}\fullonly{See the appendix for a proof.}  Now, suppose we are using
fixed precision numbers to store our model coefficients such as the
{\tt Qn.m} encoding described earlier, with a precision of $\epsilon$.
This induces a grid of feasible model coefficient vectors.  If we
randomly round each coefficient $\beta_i$ (where $\abs{\beta_i} \le
2^{\tt n}$) independently up or down to the nearest feasible value
$\hat{\beta}_i$, such that $\E[\hat{\beta}_i] = \beta_i$, then for any
$x \in \set{0,1}^{\dim}$ our predicted log-odds ratio, $\hat{\beta}
\cdot x$ is distributed as a sum of independent random variables
$\set{\hat{\beta}_i \mid x_i = 1}$.

Let $k = \norm{x}{0}$.  In this situation, note that $|\beta \cdot x -
\hat{\beta} \cdot x| \le \epsilon \nrm{x}_1 = \epsilon k$, since
$|\beta_i -\hat{\beta}_i| \le \epsilon$ for all $i$.  Thus
Lemma~\ref{lem:fixed-error} implies
\[
\logloss{x,y;\hat{\beta}} - \logloss{x,y;\beta} \le \epsilon\, \nrm{x}_1.
\]
Similarly, Lemma~\ref{lem:relative-error} immediately provides an
upper bound of $e^{\eps k} - 1$ on relative logistic error; this
bound is relatively tight for small $k$, and holds with probability
one, but it does not exploit the fact that the randomness is unbiased
and that errors should cancel out when $k$ is large.  The following
theorem gives a bound on expected relative error that is much tighter
for large $k$:

\begin{theorem} \label{thm:rounding-error} Let $\hat{\beta}$ be a
  model obtained from $\beta$ using unbiased randomized rounding to a
  precision $\epsilon$ grid as described above.
  Then, the expected logistic--loss relative error of $\hat{\beta}$ on
  any input $x$ is at most $2 \sqrt{2 \pi k}\,
  \exp\paren{\epsilon^2 k / 2} \epsilon $ where $k = \norm{x}{0}$.
\end{theorem}

\paragraph{Additional Compression} Figure~\ref{weights} reveals that
coefficient values are not uniformly distributed.  Storing these
values in a fixed-point representation means that individual values
will occur many times.  Basic information theory shows that
the more common values may be encoded with fewer bits.
The theoretical bound for a whole model with $d$ coefficients
is $\frac{-\sum_{i=1}^d \log p(\beta_i)}{d}$ bits per value, where
$p(v)$ is the probability of occurrence of $v$ in $\beta$ across all
dimensions $d$.  Variable length encoding schemes may approach this limit
and achieve further RAM savings.

\begin{table}
  \caption{
    Rounding at Prediction Time for CTR Data.  Fixed-point encodings are
    compared to a 32-bit floating point control model.  Added loss is
    negligible even when using only 1.5 bits per value with
    optimal encoding.}
 \label{prediction}
\vspace{0.1in}
\centering
\begin{small}
\renewcommand{\arraystretch}{1.2}
\begin{tabular}{ccc}
\hline
\sc  Encoding  &  AucLoss  &  Opt. Bits/Val \\
\hline
\sc q2.3            & +5.72\%    &  0.1  \\
\sc q2.5            & +0.44\%    &  0.5  \\
\sc q2.7            & +0.03\%    &  1.5  \\
\sc q2.9            & +0.00\%    &  3.3  \\
\hline
\end{tabular}
\end{small}
\end{table}

\section{Experimental Results}
We evaluated on both public and private large data sets.  We used the
public RCV1 text classification data set, specifically from
\citet{libsvm}.
In keeping with common practice on this data set, the
smaller ``train'' split of 20,242 examples was used for parameter
tuning and the larger ``test'' split of 677,399 examples was used for
the full online learning experiments.  We also report results from
a private CTR data set of roughly 30M examples and 20M features, sampled
from real ad click data from a major search engine.
Even larger experiments were run on data sets of billions of examples
and billions of dimensions, with similar results as those reported here.

The evaluation metrics for predictions are error rate for the RCV1
data, and AucLoss (or 1-AUC) relative to a control model for the CTR
data.  Lower values are better.  Metrics are computed using
progressive validation~\cite{blum:1999} as is standard for online
learning: on each round a prediction is made for a given example and
record for evaluation, and only after that is the model allowed to
train on the example.  We also report the number of bits per
coordinate used.

\paragraph{Rounding During Training} Our main results are given in
Figure~\ref{fig:training-results}.  The comparison baseline is online
logistic regression using a single global learning rate and 32-bit
floats to store coefficients.  We also test the effect of
per-coordinate learning rates with both 32-bit integers for exact
counts and with 8-bit randomized counts.  We test the range of
tradeoffs available for fixed-precision rounding with randomized
counts, varying the number of precision {\tt m} in {\tt q2.m} encoding
to plot the tradeoff curve (cyan).  We also test the range of
tradeoffs available for adaptive-precision rounding with randomized
counts, varying the precision scalar $\gam$ to plot the tradeoff curve
(dark red).  For all randomized counts a base of 1.1 was used.  Other
than these differences, the algorithms tested are identical.

Using a single global learning rate, a fixed {\tt q2.13} encoding
saves 50\% of the RAM at no added loss compared to the baseline.  The
addition of per-coordinate learning rates gives significant improvement in
predictive performance, but at the price of added memory consumption,
increasing from 32 bits per coordinate to 64 bits per coordinate in
the baselines.  Using randomized counts reduces this down to 40 bits
per coordinate.  However, both the fixed-precision and the adaptive
precision methods give far better results, achieving the same
excellent predictive performance as the 64-bit method with 24 bits per
coefficient or less.  This saves 62.5\% of the RAM cost compared
to the 64-bit method, and is still smaller than using 32-bit floats
with a global learning rate.

The benefit of adaptive precision is only apparent on the larger CTR
data set, which has a ``long tail'' distribution of support across
features.
However, it is useful to note that the simpler
fixed-precision method also gives great benefit.  For example,
using {\tt q2.13} encoding for coefficient values and 8-bit randomized
counters allows full-byte alignment in naive data structures.

\paragraph{Rounding at Prediction Time} We tested the
effect of performing coarser randomized rounding of a fully-trained
model on the CTR data, and compared to the loss incurred using a 32-bit floating point
representation.
These results, given in Table~\ref{prediction},
clearly support the theoretical analysis that suggests more aggressive
rounding is possible at prediction time.  Surprisingly coarse levels
of precision give excellent results, with little or no loss in
predictive performance.  The memory savings achievable in this scheme
are considerable, down to {\em less than two bits per value} for {\tt q2.7}
with theoretically optimal encoding of the discrete values.

\section{Conclusions}
Randomized storage of coefficient values provides an efficient method
for achieving significant RAM savings both during training and at
prediction time.

While in this work we focus on OGD, similar randomized rounding
schemes may be applied to other learning algorithms.  The extension to
algorithms that efficiently handle $L_1$ regularization, like
RDA~\citep{xiao09dualaveraging} and
FTRL-Proximal~\citep{mcmahan:2011}, is relatively
straightforward.\footnote{Some care must be taken to store a
  discretized version of a scaled gradient sum, so that the dynamic
  range remains roughly unchanged as learning progresses.}  Large
scale kernel machines, matrix decompositions, topic models, and other
large-scale learning methods may all be modifiable to take advantage
of RAM savings through low precision randomized rounding methods.

\fullonly{
\section*{Acknowledgments}
We would like to thank Matthew Streeter, Gary Holt, Todd Phillips, and
Mark Rose for their help with this work.
}

\appendix
\section{Appendix: Proofs}

\subsection{Proof of Theorem~\ref{thm:ogd}}
Our analysis extends the technique of~\citet{zinkevich03}.  Let
$\beta^*$ be any feasible point (with possibly infinite precision
coefficients).  By the definition of $\beta\ti$, %
\begin{equation*}
\nrm{\beta\ti - \beta^*}^2 = \nrm{\hb_t - \beta^*}^2
  - 2\eta_t g_t \cdot (\hb_t - \beta^*) + \eta^2_t \nrm{g_t}^2.
\end{equation*}
Rearranging the above yields
\begin{align*} \label{eqn:zinkevich1}
 g_t &\cdot (\hb_t - \beta^*) \notag\\
  &\le \frac{1}{2 \eta_t} \paren{
  \nrm{\hb_t - \beta^*}^2 - \nrm{\beta\ti - \beta^*}^2 }
   + \frac{\eta_t}{2} \nrm{g_t}^2 \\
  &= \frac{1}{2 \eta_t} \paren{ \nrm{\hb_t - \beta^*}^2 -
    \nrm{\hb\ti - \beta^*}^2 }  \notag
  + \frac{\eta_t}{2} \nrm{g_t}^2 + \rho_t,
\end{align*}
where the $\rho_t = \frac{1}{2 \eta_t} \paren{\nrm{\hb\ti - \beta^*}^2 -
  \nrm{\beta\ti - \beta^*}^2 }$ terms will capture the extra regret due
to the randomized rounding.  Summing over $t$, and following
Zinkevich's analysis, we obtain a bound of
\begin{equation*}
  \Regret(T) \le
    \frac{(2R)^2}{2 \eta_T} + \frac{\norm{g_t}{2}^2}{2} \eta_{1:T}  + \rho_{1:T}.
\end{equation*}
It remains to bound $\rho_{1:T}$.
\newcommand{\dd}{d}
\newcommand{\nd}{a}
Letting $\dd_t = \beta\ti - \hb\ti$ and $\nd_t = \dd_t/\eta_t$, we
have
\begin{align*}
\rho_{1:T}
  &= \sum_{t=1}^{T} \frac{1}{2 \eta_t}\big(
      (\hb\ti - \bs)^2 - (\beta\ti - \bs)^2 \big)\\
  &\le \sum_{t=1}^{T}  \frac{1}{2 \eta_t}\left(
      \hb\ti^2  -\beta\ti^2 \right)
       + \bs \nd_{1:T}\\
  &\le \sum_{t=1}^{T}  \frac{1}{2 \eta_t}\left(
      \hb\ti^2  -\beta\ti^2 \right)
       + R \babs{\nd_{1:T}}.
\end{align*}
We bound each of the terms in this last expression in expectation.
First, note $\abs{\dd_t} \leq \eps_t \le  \gam
\eta_t$ by definition of the resolution of the rounding grid, and so
$\abs{\nd_t} \le \gam$. Further $\E[d_t] = 0$ since the rounding is
unbiased.  Letting $W=\abs{\nd_{1:T}}$, by Jensen's inequality we have
$\E[W]^2 \le \E[W^2]$.  Thus, $ \E[\abs{\nd_{1:T}}] \le
\sqrt{\E[(\nd_{1:T})^2]} = \sqrt{\var(\nd_{1:T})},$ where the last
equality follows from the fact $\E[\nd_{1:T}] = 0$.
The $\nd_t$ are not independent given an adaptive
adversary.\footnote{For example the adversary could ensure $\nd_{t+1}
  = 0$ (by playing $g_{t+1} = 0$) iff $\nd_t > 0$.}
Nevertheless, consider any $\nd_s$ and $\nd_t$ with $s < t$.  Since both
have expectation zero, $\cov(\nd_s, \nd_t) = \E[\nd_s \nd_t]$.  By
construction, $\E[\nd_t \mid g_t, \beta_t, \text{hist}_t] = 0$, where
$\text{hist}_t$ is the full history of the game up until round $t$,
which includes $\nd_s$ in particular.  Thus
\[ \cov(\nd_s, \nd_t)
 = \E[\nd_s \nd_t]
 = \E\big[ \E[\nd_s \nd_t \mid g_t, \beta_t, \text{hist}_t]\big] = 0.
\]
For all $t$, $\abs{\nd_t} \leq \gam$ so $\var(\nd_t) \leq \gam^2$, and
$\var(\nd_{1:T}) = \sum_t \var(\nd_t) \leq \gam^2 T$.  Thus,
$\E[\abs{\nd_{1:T}}] \leq \gam \sqrt{T}$.

Next, consider $\E[\hb\ti^2 - \beta\ti^2 \mid \beta\ti]$.  Since
$\E[\hb\ti \mid \beta\ti] = \beta\ti$, for any shift $s \in \R$, we
have
$
\E\big[(\hb\ti - s)^2 - (\beta\ti - s)^2 \mid \beta\ti\big]
 = \E\big[\hb\ti^2 - \beta\ti^2 \mid \beta\ti\big],
$
and so taking $s= \beta\ti$,
\begin{align*}
 \frac{1}{\eta_t}\E\big[\hb\ti^2 - \beta\ti^2 \mid \beta\ti\big]
  &=  \frac{1}{\eta_t}\E\big[(\hb\ti - \beta\ti)^2 \mid \beta\ti\big] \\
  &\leq \frac{\eps_t^2}{\eta_t} \leq \frac{\gam^2 \eta_t^2}{\eta_t}
    = \gam^2 \eta_t.
\end{align*}
Combining this result with $\E[\abs{\nd_{1:T}}] \leq \gam \sqrt{T}$,
we have
\[
 \expct{\rho_{1:T}} \le \gam^2 \eta_{1:T} + \gam R \sqrt{T},
\]
which completes the proof. \qed

\subsection{Approximate Counting}

We first provide high--probability bounds for the approximate counter.

\begin{lemma}\label{lem:approx-counting}
  Fix $T$ and $t \le T$.  Let $\counter\ti$ be the value of the
  counter after $t$ increment operations using the approximate
  counting algorithm described in Section~\ref{sec:approxcounts} with
  base $\base > 1$.  Then, for all $c > 0$, the estimated count
  $\est(\counter\ti)$ satisfies
  \begin{equation}\label{eq:b1}
    \prob{ \est(\counter\ti) < \frac{t}{\base c \log(T)} - 1} \ \le\
    \frac{1}{T^{c-1}}
  \end{equation}
  and
  \begin{equation}\label{eq:b2}
  \prob{ \est(\counter\ti) >
    \frac{et}{\base-1} \base^{\sqrt{2c\log_{\base}(T)}+2} }
    \ \le\  \frac{1}{T^{c}}.
  \end{equation}
\end{lemma}
Both $T$ and $c$ are essentially parameters of the bound; in
the \eqr{b2}, any choices of $T$ and $c$ that keep $T^c$ constant
produce the same bound.  In the first bound, the result is sharpest
when $T = t$, but it will be convenient to set $T$ equal to the total
number of rounds so that we can easily take a union bound (in the proof of Theorem~\ref{thm:ogd-approx-count}).

\begin{proof}[Proof of Lemma~\ref{lem:approx-counting}]
  Fix a sequence of $T$ increments, and let $\counter_i$ denote the
  value of the approximate counter at the start of increment number
  $i$, so $\counter_1 = 1$.  Let $X_j = |\set{i : \counter_i = j}|$, a
  random variable for the number of increments for which the
  counter stayed at $j$.

  We start with the bound of~\eqr{b1}.
  When $\counter = j$, the update probability is $p_j = \incprob(j) =
  \base^{-j}$, so for any $\ell_j$ we have $X_j \ge \ell_j$ with
  probability at most $(1 - p_j)^{\ell_j} \le \exp(-p_j)^{\ell_j} =
  \exp(- p_j \ell_j)$ since $(1 - x) \le \exp(-x)$ for all $x$.
  To make this at most $T^{-c}$ it suffices to take $\ell_j = c (\log
  T) / p_j = c \base^{j} \log T $.  Taking a (rather loose) union
  bound over $j = 1, 2, \ldots, T$, we have
  \[\prob{\exists j,\ X_j > c \base^{j} \log T  } \le 1/T^{c-1}.\]

  For~\eqr{b1}, it suffices to show that if this does not occur, then
  $\est(\counter_t) \ge t / (\base c \log(T)) - 1$.  Note
  $\sum_{j=1}^{\counter_t} X_j \ge t$.  With our supposition that $X_j
  \le c \base^{j} \log T$ for all $j$, this implies $t \le
  \sum_{j=1}^{\counter_t} c \base^{j} \log T = c \base \log T
  \paren{\frac{\base^{\counter_t} - 1}{\base - 1}}$, and thus
  $\counter_t \ge \log_{\base}\paren{\frac{t(\base-1)}{\base c \log T}
    + 1}$.  Since $\est$ is monotonically increasing and $b>1$, simple
  algebra then shows $\est(\counter\ti) \ge \est(\counter_t) \ge t /
  (\base c \log(T)) - 1$.

  Next consider the bound of ~\eqr{b2}.  Let $j_0$ be the minimum
  value such that $\incprob(j_0) \le 1/et$, and fix $k \ge 0$.
  Then $\counter\ti \ge j_0 + k$ implies the counter was incremented
  $k$ times with an increment probability at most $p(j_0)$.  Thus,
  \begin{align*}
   \prob{\counter_t \ge j_0 + k}
     &\le {t \choose k} \prod_{j=j_0}^{j_0+k-1} \incprob(j) \notag \\
     &\le \paren{\frac{te}{k}}^k
         \paren{\prod_{j=0}^{k-1} \incprob(j_0)\base^{-j}} \\
     &= \paren{\frac{te}{k}}^k {\incprob(j_0)}^k \base^{-k(k-1)/2} \\
     &\le k^{-k} \cdot \base^{-k(k-1)/2}
  \end{align*}
  Note that $j_0 \le \ceil{\log_{\base} \paren{et}}$.  Taking $k =
  \sqrt{2c\log_{\base}(T)} + 1$ is sufficient to ensure this
  probability is at most $T^{-c}$, since $k^{-k} \leq 1$ and $k^2 -k
  \geq 2 c \log_{\base} T$.
  Observing that $\est \paren{\ceil{\log_{\base} \paren{et}} +
    \sqrt{2c\log_{\base}(T)} + 1} \le \frac{et}{\base-1}
  \base^{\sqrt{2c\log_{\base}(T)}+2}$ completes the proof.
\end{proof}

\paragraph{Proof of Theorem~\ref{thm:ogd-approx-count}.}
\newcommand{\ka}{k_1}
\newcommand{\kb}{k_2}

We prove the bound for the one-dimensional case; the general bound
then follows by summing over dimensions.  Since we consider a single
dimension, we assume $\abs{g_t} > 0$ on all rounds.  This is without loss
of generality, because we can implicitly skip all rounds with zero
gradients, which means we don't need to make the distinction between
$t$ and $\cnt_{t,i}$.  We abuse notation slightly by defining $\acnt_t
\defeq \est(C\ti) \approx t = \cnt_t$ for the approximate count on
round $t$.
We begin from the bound
\[ \E[\Regret] \le
\frac{(2R)^2}{2 \eta_{T}} + \h (G^2 + \gam^2)\eta_{1:t}
+ \gam R \sqrt{T}.
\]
of Theorem~\ref{thm:ogd}, with learning rates $\eta_t
=\alpha/\sqrt{\acnt_t + 1}$.
Lemma~\ref{lem:approx-counting} with $c=2.5$ then implies
\[
\prob{\acnt_t + 1< \ka t} \le \frac{1}{T^{1.5}} \ \text{ and }\
\prob{\acnt_t > \kb t}  \le \frac{1}{T^{2.5}},
\]
where $\ka = 1/(\base c \log T)$ and $\kb = \frac{e \base^{\sqrt{2 c
      \log_{\base}T} + 2}}{\base - 1}$.
A union bound on $t=1, ..., T$ on the first bound implies with
probability $1 - \frac{1}{\sqrt{T}}$ we have $\forall t,\ \acnt_t +1
\ge \ka t$, so
\begin{equation}  \label{eq:etaA}
  \eta_{1:T}
  = \sum_{t=1}^T \frac{\alpha}{\sqrt{\acnt_t + 1}}
  \le  \frac{1}{\sqrt{\ka}}\sum_{t=1}^T \frac{\alpha}{\sqrt{t}}
  \le \frac{2\alpha \sqrt{T}}{\sqrt{\ka}},
\end{equation}
where we have used the inequality $\sum_{t=1}^T \frac{1}{\sqrt{t}}
\le 2 \sqrt{T}$.  Similarly, the second inequality implies with
probability at least $1 - \frac{1}{T^{2.5}}$,
\begin{equation}\label{eq:etaB}
  \eta_T = \frac{\alpha}{\sqrt{\acnt_T + 1}}
  \ge \frac{\alpha}{\sqrt{\kb T + 1}}.
\end{equation}
Taking a union bound, Eqs.~\eqref{eq:etaA} and \eqref{eq:etaB} hold
with probability at least $1 - 2/\sqrt{T}$, and so at least one fails
with probability at most $2/\sqrt{T}$.  Since $f_t(\beta) -
f_t(\beta') \leq 2GR$ for any $\beta, \beta' \in [-R, R]$ (using the
convexity of $f_t$ and the bound on the gradients $G$), on any run of
the algorithm, regret is bounded by $2RGT$.  Thus, these failed cases
contribute at most $4 RG\sqrt{T}$ to the expected regret bound.

Now suppose Eqs.~\eqref{eq:etaA} and \eqref{eq:etaB} hold.  Choosing
$\alpha = \frac{R}{\sqrt{G^2 + \gam^2}}$ minimizes the dependence on the
other constants, and note for any $\delta > 0$, both
$\frac{1}{\sqrt{\ka}}$ and $ \sqrt{\kb}$ are $o(T^\delta)$.
Thus, when Eqs.~\eqref{eq:etaA} and \eqref{eq:etaB} hold,
\begin{align*}
  \E[\Regret]
  &\le \frac{(2R)^2}{2 \eta_{T}} + \h (G^2 + \gam^2)\eta_{1:t}
  + \gam R \sqrt{T}  \\
  &\le \frac{2R^2\sqrt{\kb T + 1}}{\alpha} + (G^2 + \gam^2)
  \frac{\alpha \sqrt{T}}{\sqrt{\ka}}
  + \gam R \sqrt{T} \\
  &= o\Big(R \sqrt{G^2 + \gamma^2} T^{0.5 + \delta}\Big).
\end{align*}
Adding $4RG\sqrt{T}$ for the case when the high-probability
statements fail still leaves the same bound.  \qed

It follows from the proof that we have the more precise but cumbersome
upper bound on $\E[\Regret]$:
\begin{equation}\label{eq:fullregret}
  \frac{2R^2\sqrt{\kb T+1}}{\alpha} + (G^2 + \gam^2)
  \frac{\alpha \sqrt{T}}{\sqrt{\ka}}
  + \gam R \sqrt{T}
  + 4 R G\sqrt{T}.
\end{equation}

\fullonly{
\subsection{Encoding During Prediction Time}

We use the following well--known inequality, which is a direct
corollary of the Azuma--Hoeffding inequality.  For a proof,
see~\cite{chung06}.
\begin{theorem} \label{thm:azuma} Let $X_1, \ldots, X_{\dim}$ be
  independent random variables such that for each $i$, there is a
   constant $c_i$ such that $|X_i - \expct{X_i}| \le c_i$, always.  Let
   $X = \sum_{i=1}^{\dim} X_i$.  Then $\prob{|X - \expct{X}| \ge t} \le
   2 \exp \set{ -t^2 / 2 \sum_i c_i^2 }$.
 \end{theorem}

 An immediate consequence is the following large deviation bound on
 $\delta = |\beta \cdot x - \hat{\beta} \cdot x|$:
\begin{lemma} \label{lem:delta-distribution} Let $\hat{\beta}$ be a
  model obtained from $\beta$ using unbiased randomized rounding to a
  precision $\epsilon$ grid.  Fix $x$, and let $Z = \hat{\beta} \cdot
  x$ be the random predicted log-odds ratio.  Then
  \[
  \prob{|Z - \beta \cdot x| \ge t} \le 2 \exp \paren{ \frac{-t^2} { 2
    \epsilon^2 \norm{x}{0} } }
  \]
\end{lemma}

Lemmas~\ref{lem:fixed-error} and~\ref{lem:relative-error} provide
bounds in terms of the quantity $|\beta \cdot x - \hat{\beta} \cdot
x|$.  The former is proved in Section~\ref{sec:prediction}; we now
provide a proof of the latter.

\newcommand{\predictfrommodel}[2]{\ensuremath{p_{#1}\paren{#2}}}
\newcommand{\pfm}[1]{\predictfrommodel{\hat{\beta}}{#1}}
\newcommand{\rmloss}[0]{\hat{\alpha}}
\newcommand{\mloss}[0]{\alpha}

\paragraph{Proof of Lemma~\ref{lem:relative-error}}
\newcommand{\z}{z}
\newcommand{\zh}{\hat{z}}
We claim that the relative error is bounded as
\begin{equation}\label{eq:claim}
\frac{\logloss{x,y;\hat{\beta}} - \logloss{x,y;\beta}}{\logloss{x,y;\beta}}
  \le e^{\delta} - 1,
\end{equation}
or equivalently, that that $\logloss{x,y;\hat{\beta}} \le e^{\delta}
\logloss{x,y; \beta}$, where $\delta \defeq |\beta \cdot x -
\hat{\beta} \cdot x|$ as before.
We will argue the case in which $y = 1$; the $y = 0$ case is
analogous.  Let $z = \beta \cdot x$, and $\hz = \hb \cdot x$; then,
when $y=1$,
\[
 \logloss{x,y, \beta} = \log(1 + \exp(-\z)),
\]
and similarly for $\hb$ and $\hz$.  If $\hz > z$ then
$\logloss{x,y;\hat{\beta}}$ is less than $\logloss{x,y;\beta}$, which
immediately implies the claim.  Thus, we need only consider the case
when $\hz = \z - \delta$.  Then, the claim of \eqr{claim} is
equivalent to
\begin{align*}
\log \paren{1 + \exp\paren{-z + \delta}}
   &\le  \exp(\delta) \log \paren{1 + \exp\paren{-z}},\\
\intertext{or equivalently,}
 1 + \exp\paren{-z + \delta}
   &\le \paren{1 + \exp\paren{-z}}^{\exp(\delta)}.
\end{align*}
Let $w \defeq \exp\paren{\delta}$ and $u \defeq \exp\paren{-z}$.
Then, we can rewrite the last line as $1 + wu \le (1 + u)^w$, which is
true by Bernoulli's inequality, since $u \ge 0$ and $w \ge 1$. \qed

\paragraph{Proof of Theorem~\ref{thm:rounding-error}}
Let $R = \frac{\logloss{x,y;\hat{\beta}} -
  \logloss{x,y;\beta}}{\logloss{x,y;\beta}}$ denote the relative error
due to rounding, and let $R(\delta)$ be the worst case expected
relative error given $\delta = |\hat{\beta}\cdot x - \beta \cdot x|$.
\newcommand{\Y}{\bar{R}}
\newcommand{\y}{r}
Let $\Y \equiv e^\delta - 1$.  Then, by Lemma~\ref{lem:relative-error},
$R(\delta) \le \Y(\delta)$.  It is sufficient to prove a suitable upper bound
on $\expct{\Y}$.  First, for $\y \geq 0$,
\begin{align*}
\prob{\Y \ge \y}
  &= \prob{e^\delta - 1 \ge \y} \\
  &= \prob{\delta \ge \log(\y+1) } \\
  &\le 2 \exp \paren{ \frac{-\log^2 (\y+1)} { 2 \epsilon^2 \norm{x}{0} }}.
   &&  \text{[Lemma~\ref{lem:delta-distribution}]}
\end{align*}
Using this, we bound the expectation of $\Y$ as follows:
\begin{align*}
\E[\Y] &= \int_{\y=0}^\infty \prob{\Y \ge \y} d\y \\
      &\le 2 \int_{\y=0}^\infty
        \exp \paren{ \frac{-\log^2 (\y+1)} { 2 \epsilon^2 \norm{x}{0} }} d\y, \\
        \intertext{and since the function being integrated is
          non-negative on $(-1, \infty)$,}
      &\le 2\int_{\y=-1}^\infty
        \exp \paren{ \frac{-\log^2 (\y+1)} { 2 \epsilon^2 \norm{x}{0} }} d\y \\
      &= 2 \sqrt{2 \pi \norm{x}{0}}
      \exp\paren{\frac{\eps^2 \norm{x}{0}}{2}} \epsilon,
\end{align*}
where the last line follows after straightforward calculus.  A
slightly tighter bound (replacing the leading $2$ with $1 + \Erf(\eps
\sqrt{\norm{x}{0}}/\sqrt{2})$) can be obtained if one does not make the
change in the lower limit of integration.\qed } %

\clearpage
\begin{small}
\bibliography{round-model}
\bibliographystyle{icml2013}
\end{small}

\end{document}